\documentclass[final]{siamltex}
\usepackage{graphicx}
\usepackage{mathtools}
\usepackage{latexsym}
\usepackage[mathscr]{euscript}
\usepackage{xcolor}
\usepackage{algorithm,algorithmic}
\usepackage{fancyhdr}
\usepackage{multirow}
\usepackage{epsfig}
\usepackage{mathptmx}
\usepackage[most]{tcolorbox}
\usepackage{helvet}
\usepackage{courier}
\usepackage{type1cm}
\usepackage{imakeidx}
\usepackage[a4paper,textheight=23cm,textwidth=15cm]{geometry}
\usepackage{hyperref}
\usepackage{enumitem}
\usepackage{titlesec}
\usepackage{amsmath,amssymb,amsfonts}

\usepackage{physics}
\usepackage{tikz}
\usepackage{bm}
\usepackage{booktabs}
\usetikzlibrary{arrows.meta,shapes,positioning,calc,fit,backgrounds,decorations.pathreplacing}
\usepackage{pifont}

\title{
	{Multidimensional Task Learning: A Unified Tensor Framework for Computer Vision Tasks }}
	\author{A. Elichi\thanks{Université du Littoral Cote d'Opale, LMPA, 50 rue F. Buisson, 62228 Calais-Cedex, France.}   \and K. Jbilou\footnotemark[1] .} 

\date{}

\begin{document}
	\maketitle
	
	\begin{abstract}
This paper introduces Multidimensional Task Learning (MTL), a unified mathematical framework based on Generalized Einstein MLPs (GE-MLPs) that operate directly on tensors via the Einstein product. We argue that current computer vision task formulations are inherently constrained by matrix-based thinking: standard architectures rely on matrix-valued weights and vector-valued biases, requiring structural flattening that restricts the space of naturally expressible tasks. GE-MLPs lift this constraint by operating with tensor-valued parameters, enabling explicit control over which dimensions are preserved or contracted without information loss. Through rigorous mathematical derivations, we demonstrate that classification, segmentation, and detection are special cases of MTL, differing only in their dimensional configuration within a formally defined task space $\mathcal{S}_{\mathrm{MTL}}$. We further prove that this task space is strictly larger than what matrix-based formulations can natively express, enabling principled task configurations---such as spatiotemporal or cross-modal predictions---that require destructive flattening under conventional approaches. This work provides a mathematical foundation for understanding, comparing, and designing computer vision tasks through the lens of tensor algebra.
		
	\end{abstract}

	\section{Introduction}
	\label{sec:introduction}
	
{C}{omputer} vision tasks exhibit fundamental differences in their formulations. Image classification assigns a single label per image \cite{krizhevsky2012,he2016}, semantic segmentation produces per-pixel labels \cite{long2015,chen2018}, and object detection predicts multiple structured outputs per spatial region \cite{redmon2016,ren2015}. These differences manifest in separate architectural designs (ResNet vs FCN vs YOLO), distinct loss functions (categorical cross-entropy vs dense pixel-wise losses vs multi-task detection losses), and specialized training procedures.

\noindent We observe that all these tasks operate on tensor-structured data and differ fundamentally in two aspects:
(i) which structural dimensions are preserved during processing, and (ii) how many output modalities are predicted. For instance, classification preserves only the batch dimension, producing $(B, C_{\text{classes}})$ outputs. Segmentation preserves spatial structure, yielding $(B, H, W, C_{\text{classes}})$ outputs. Detection preserves grid structure while predicting three modalities (bounding box, objectness, class), resulting in $(B, G_h, G_w, 4+1+C)$ outputs. \\
\noindent  To address this problem, we introduce Multidimensional Task Learning 
(MTL), a unified mathematical framework based on Generalized Einstein MLPs 
(GE-MLPs) that operate directly on high-dimensional tensors via the Einstein 
product \cite{brazell2013}. GE-MLPs manipulate tensor weights and biases 
instead of traditional matrices and vectors, contracting user-specified axes 
while preserving others, achieving precise control over output dimensions.\\
Within the MTL framework, traditional computer vision tasks emerge as special cases distinguished by their dimensional configurations. This unified approach not only recovers existing task formulations as special 
cases but also reveals unexplored task configurations. By parameterizing tasks 
through dimensional configuration  and introducing the structure 
preservation index $\rho \in [0,1]$.\\
\noindent The main contributions of this paper are:
\begin{itemize}
	\item \textbf{Multidimensional Task Learning (MTL):} A unified mathematical 
	framework parameterizing tasks by tuple $\mathcal{T} = (P, M, \mathcal{L}, \phi)$ 
	where $P$ denotes output dimensions and $M$ preserved structural dimensions, 
	revealing that task differences reduce to dimensional configuration choices.
	
	\item \textbf{Generalized Einstein MLPs (GEMLPs):} Tensor-based 
	architecture operating directly on high-dimensional tensors via Einstein 
	product, 
	eliminating flattening operations while achieving computational complexity 
	identical to specialized architectures.
	
	\item \textbf{Structure Preservation Index:} Introduction of 
	$\rho   \in [0,1]$ quantifying preservation information .

	\item \textbf{Theoretical Unification:} Complete proofs with rigorous 
	mathematical derivations establishing that classification, segmentation, 
	and detection correspond to specific configurations.

\end{itemize}

\noindent This work provides the first formulation-level unification of computer vision 
tasks, demonstrating that apparent architectural differences reduce to choices 
in which dimensions to preserve versus contract during computation. By 
establishing MTL as a foundational framework, we enable both deeper 
understanding of existing tasks and systematic creation of new ones.\\

\noindent \noindent The remainder of this paper is organized as follows. 
Section \ref{background} introduces the mathematical foundations 
including tensor operations and Generalized Einstein MLPs. 
Section \ref{sec:unif} presents the theoretical unification 
with complete proofs that classification, segmentation, and detection 
are special cases of MTL. 
Section \ref{sec:dissc} discusses implications and novel task formulations enabled by the framework. 
Section \ref{sec:conc} concludes the paper.

\section{Methodology: GE-MLP}\label{background}
\noindent In this section, we introduce the mathematical background and notations necessary to understand the Generalized Einstein MLPs (GEMLPs).\\
\noindent  	A multidimensional array of data is called a tensor. We refer to a tensor's number of indices as its order, modes or ways. 
Notice that a scalar is  zero mode tensor, first and second mode tensors are vectors and matrices respectively. 
For a given N-mode (or order-N) tensor $ \mathcal {X}\in \mathbb{R}^{I_{1}\times \ldots \times I_{N}}$, the notation $x_{i_{1},\ldots,i_{N}}$ (with $1\leq i_{j}\leq I_{j}$ and $ j=1,\ldots N $) stands for the element $\left(i_{1},\ldots,i_{N} \right) $ of the tensor $\mathcal {X}$. \\
\noindent In the following, we recall some useful tensor product:
\begin{itemize}
	\item \textbf{Tensor Norm:}
	The norm of a tensor $\mathcal{A}\in \mathbb{R}^{I_1\times \cdots \times I_N}$ is specified by
	\begin{equation}  
		\left\| \mathcal{A} \right\|_F^2 = {\sum\limits_{i_1  = 1}^{I_1 } {\sum\limits_{i_2  = 1}^{I_2 } {\cdots\sum\limits_{i_\ell = 1}^{I_N} {a_{i_1 i_2 \cdots i_\ell }^2 } } } }^{}
		\label{eq:normf}   
	\end{equation}
	where $	\left\|  . \right\|_F$ is the Frobenius norm and $a_{i_1 i_2 \cdots i_\ell} $ represents the entry of $\mathcal{A}$ at the index $(i_1, i_2 \cdots ,i_\ell)$. 
	\item \textbf{Einstein product}\cite{brazell2013}:
	The Einstein product  between two tensors $\mathcal {X}\in \mathbb{R}^{I_{1}\times  \ldots \times I_{N}\times J_{1}\times  \ldots \times J_{M}}$ and $\mathcal {Y}\in \mathbb{R}^{J_{1}\times  \ldots \times J_{M}\times K_{1}\times  \ldots \times K_{L}}$ is a tensor of size ${I_{1}\times  \ldots \times I_{N}\times K_{1}\times  \ldots \times K_{L} }$ whose elements are given by:
	\begin{equation}
		\begin{aligned}
			(\mathcal {X}\ast_M\mathcal {Y})_{i_1,\ldots,i_N,k_1,\ldots,k_L} &=  
			\sum_{j_1,\ldots,j_M}^{J_1,\ldots,J_M} 
			{x}_{i_1,\ldots,i_N,j_1,\ldots,j_M} \\
			&\quad \times {y}_{j_1,\ldots,j_M,k_1,\ldots,k_L}
		\end{aligned}
		\label{eq:einsteinprod}
	\end{equation}
\end{itemize}	 
\subsection{General Description of GEMLP}\label{descrip}
\noindent   The main idea in  Generalized Einstein MLPs is to eliminate the Flatten step.  Let a GE MLP network have \( L \) layers, indexed by \( \ell = 1, \dots, T \). For the \( \ell \)-th layer:

\begin{itemize}
	\item \textbf{Input tensor}: 
	$
	\mathcal{X}^{(\ell)} \in \mathbb{R}^{I_1 \times \dots \times I_N \times J_1 \times \dots \times J_M}
	$
	Contracting dimensions \( I_1, \dots, I_N \) (e.g., features/channels) and preserved dimensions \( J_1, \dots, J_M \) (e.g., spatial positions) ($	\mathcal{X}^{(0)}$ the input initial tensor).
	
	\item \textbf{Weight and Bias tensor}: 
	$ 
	\mathcal{W}^{(\ell)} \in \mathbb{R}^{K_1 \times \dots \times K_P \times I_1 \times \dots \times I_N}
	$
	Transforms contracting dimensions \( I_1, \dots, I_N \) to \( K_1, \dots, K_P \) and  
	$
	\mathcal{B}^{(\ell)} \in \mathbb{R}^{K_1 \times \dots \times K_P \times J_1 \times \dots \times J_M}
	$ (P is flexible and user-selected ).
	
	\item \textbf{Output tensor}: 
	$
	\mathcal{Y}^{(\ell)} \in \mathbb{R}^{K_1 \times \dots \times K_P \times J_1 \times \dots \times J_M}
	$
	Preserves \( J_1, \dots, J_M \).
	
\end{itemize}
\noindent The output at layer \( \ell \) is computed via tensor contraction and activation:
\begin{equation}
	\mathcal{Y}^{(\ell)} = f\left(   \mathcal{W}^{(\ell)} \ast_N \mathcal{X}^{(\ell-1)} + \mathcal{B}^{(\ell)} \right)
\end{equation}

\noindent Explicitly, for each position \( (j_1, \dots, j_M) \) in preserved dimensions:
\begin{equation}
	\begin{aligned}
		y^{(\ell)}_{k_1, \dots, k_P, j_1, \dots, j_M} &= f\Biggl( \underbrace{\sum_{i_1=1}^{I_1} \cdots \sum_{i_N=1}^{I_N}}_{\text{Contraction}} w^{(\ell)}_{k_1, \dots, k_P, i_1, \dots, i_N} \nonumber \\
		&\quad  x^{(\ell-1)}_{i_1, \dots, i_N, j_1, \dots, j_M} + b^{(\ell)}_{k_1, \dots, k_P, j_1, \dots, j_M} \Biggr)
	\end{aligned}\label{eq:outputlayer}
\end{equation}
where $f$ is the activation function.\\
\noindent To adjust the weight  tensor $\mathcal {W} \in \mathbb{R}^{K_1\times \ldots \times K_P\times  I_{1}\times \ldots \times  I_{N}}$ in order to minimize a loss function $L(\mathcal{Y}, \mathcal{Y}_{\text{True}})$, we introduce the Generalized Einstein Gradient Descent (GEGD). This process can be described as follows:
\begin{itemize}
	\item  \textbf{Gradient Computation}: The gradients of the weights and bias are given by:
	\begin{equation}
		\frac{\partial L}{\partial w_{k_1,\ldots,k_P,i_1,\ldots,i_N}} = \frac{\partial L}{\partial y_{k_1,\ldots,k_P,j_1,\ldots,j_M}} \cdot x_{i_1,\ldots,i_N,j_1,\ldots,j_M},
	\end{equation}
	where $x_{i_1,\ldots,i_N,j_1,\ldots,j_M}$ represents the input tensor element corresponding to the weight $w_{k_1,\ldots,k_P,i_1,\ldots,i_N}$.
	\item \textbf{Parameter Updates}:
	The weights and biases are updated using a gradient descent rule. The weight update is given by:
	\begin{equation}
		\begin{aligned}
			w_{k_1,\ldots,k_P,i_1,\ldots,i_N} \leftarrow &w_{k_1,\ldots,k_P,i_1,\ldots,i_N} - \\
			&\gamma \cdot \frac{\partial L}{\partial y_{k_1,\ldots,k_P,j_1,\ldots,j_M}} \cdot x_{i_1,\ldots,i_N,j_1,\ldots,j_M},
		\end{aligned}\label{eq:weight}
	\end{equation}
	and the bias update is given by:
	\begin{equation}
		\begin{aligned}
			b_{k_1,\ldots,k_P,j_1,\ldots,j_M} \leftarrow & b_{k_1,\ldots,k_P,j_1,\ldots,j_M} \\
			& - \gamma \cdot \frac{\partial L}{\partial b_{k_1,\ldots,k_P,j_1,\ldots,j_M}},
		\end{aligned}\label{eq:bias}
	\end{equation}
	where $\gamma > 0$ is the learning rate.
\end{itemize} 

\subsection{Complexity Analysis and Comparison}
The computational complexity of a GE-MLP depends on the dimensions involved in the tensor contractions.  
\begin{itemize}
	\item \textbf{General Complexity}:
	The complexity of a GE-MLP is given by:
	\begin{equation}
		O\left(\prod_{i=1}^N I_i \cdot \prod_{m=1}^M J_m \cdot \prod_{p=1}^P K_p\right),
	\end{equation}\label{eq:complext}
	\item \textbf{Memory  Complexity}
	The memory required for a GE-MLP layer is dominated by the weight and bias tensors:
	\begin{equation}
		O\left(\prod_{p=1}^P K_p \cdot \prod_{n=1}^N I_n + \prod_{p=1}^P K_p \cdot \prod_{m=1}^M J_m\right).
	\end{equation}
	
	\item \textbf{Number of FLOPs}:
	Each tensor contraction involves multiplications and additions. The total number of floating-point operations (FLOPs) is approximately:
	\begin{equation}
		\text{FLOPs} = 2 \cdot \prod_{i=1}^N I_i \cdot \prod_{m=1}^M J_m \cdot \prod_{p=1}^P K_p.
	\end{equation}\label{eq:flops}
\end{itemize}

\subsection{Multidimensional Task Learning (MTL)}\label{subsection:classif}

\noindent Building upon the foundation of GE-MLPs and GEGD, we introduce Multidimensional Task Learning (MTL) , a novel paradigm that generalizes traditional tasks such as classification, segmentation, and detection into higher-dimensional settings while preserving the structural integrity of the data.\\
\noindent MTL leverages the flexibility of GE-MLPs to control contracting $(K_1,\ldots,K_P)$ and preserved dimensions $(J_1,\ldots,J_M)$ , enabling seamless adaptation to complex, structured data.\\
\noindent We now present a rigorous theoretical framework for Multidimensional Task Learning (MTL). We establish formal definitions, theorems, and proofs demonstrating how traditional computer vision tasks emerge as special cases of a unified tensor-based formulation.

\subsection{Foundational Definitions}

\begin{definition}[Task tuple configuration]
	\label{def:task_config}
	A task configuration is a tuple $\mathcal{T} = (P, M, \mathcal{L}, \phi)$ where:
	\begin{itemize}
		\item $P \in \mathbb{N}$: number of output contracting dimensions
		\item $M \in \mathbb{N}$: number of preserved dimensions
		\item $\mathcal{L}$: loss function mapping predictions to scalars
		\item $\phi$: output interpretation function (e.g., argmax, threshold)
	\end{itemize}
\end{definition}

\begin{definition}[Multidimensional Task]
	Given input $\mathcal{X} \in \mathbb{R}^{I_1 \times \cdots \times I_N \times J_1 \times \cdots \times J_M}$ 
	and task configuration $\mathcal{T} = (P, M, \mathcal{L}, \phi)$, 
	a multidimensional task is defined by a function 
	$\mathcal{F}_{\mathcal{T}}: \mathcal{X} \to \mathcal{Y}$ 
	implemented via GE-MLPs, producing structured outputs:
	\begin{equation}
		\mathcal{Y} = \mathcal{F}_{\mathcal{T}}(\mathcal{X}) \in \mathbb{R}^{K_1 \times \cdots \times K_P \times J_1 \times \cdots \times J_M}
	\end{equation}
	where each preserved position $(j_1, \ldots, j_M)$ receives 
	predictions over $(K_1, \ldots, K_P)$ dimensions.
\end{definition}

\begin{definition}[Structure Preservation Index]
	\label{def:preservation_index}
	For a task $\mathcal{T}$ with preserved dimensions $M$ and input spatial/temporal dimensions $M_{\text{input}}$, the structure preservation index is:
	\begin{equation}
		\rho(\mathcal{T}) = \frac{M}{M_{\text{input}}} \in [0, 1]
	\end{equation}
	where:
	\begin{itemize}
		\item $\rho = 0$: complete contraction.
		\item $\rho = 1$: complete preservation.
		\item $0 < \rho < 1$: partial preservation.
	\end{itemize}
\end{definition}

\section{Task Unification Theorems}\label{sec:unif}

\subsection{Classification}

\begin{theorem}[Classification as Special Case]
	\label{thm:classification_recovery}
	Traditional image classification with $C$ classes is recovered by MTL with configuration:
	\begin{equation}
		\mathcal{T}_{\text{class}} = (P=1, M=1, \mathcal{L}_{\text{CE}}, \phi_{\text{argmax}})
	\end{equation}
	where $K_1 = C$, $J_1 = B$ (batch only), and $\mathcal{L}_{\text{CE}}$ is categorical cross-entropy.
\end{theorem}

\begin{proof}
	Given input $\mathcal{X} \in \mathbb{R}^{B \times H \times W \times C_{\text{in}}}$, 
	we configure the GE-MLP with $N=1$ (contracting spatial dimensions: $I_1 = HWC_{\text{in}}$), 
	$M=1$ (preserving batch: $J_1=B$), and $P=1$ (output classes: $K_1=C$). 
	The forward pass computes:
	\begin{equation}
		y_{c,b} = \sum_{i=1}^{HWC_{\text{in}}} w_{c,i} \cdot x_{i,b} + b_c
	\end{equation}
	Applying softmax yields $\hat{y}_{c,b} = \exp(y_{c,b})/\sum_{c'} \exp(y_{c',b})$, 
	with categorical cross-entropy loss 
	$$\mathcal{L}_{\text{CE}} = -\frac{1}{B}\sum_b \sum_c y^*_{c,b} \log(\hat{y}_{c,b})$$ 
	where $y^*_{c,b} \in \{0,1\}$ is one-hot encoded. 
	Prediction uses $\phi_{\text{argmax}}(\hat{y}_b) = \arg\max_c \hat{y}_{c,b}$. 
	This formulation exactly recovers standard classification 
	$(B,H,W,C_{\text{in}}) \xrightarrow{\text{flatten}} (B,D) \xrightarrow{\text{linear+softmax}} (B,C)$ 
	with structure preservation index $\rho = 1/3$ since only batch is preserved from $(B,H,W)$.
\end{proof}

\begin{theorem}[Dense Classification Extension]
	\label{thm:dense_classification}
	Dense classification (classification at each spatial position) is recovered by MTL with configuration:
	\begin{equation}
		\mathcal{T}_{\text{dense}} = (P=1, M=3, \mathcal{L}_{\text{CE-dense}}, \phi_{\text{argmax}})
	\end{equation}
	where $K_1 = C$, $\mathbf{J} = (B, H, W)$.
\end{theorem}
\begin{proof}
	For input $\mathcal{X} \in \mathbb{R}^{B \times H \times W \times C_{\text{in}}}$ 
	targeting per-pixel classification $y_{b,h,w} \in \{1,\ldots,C\}$, 
	we configure the GE-MLP to contract only channel dimensions ($N=1$, $I_1=C_{\text{in}}$) 
	while preserving full spatial structure ($M=3$, $\mathbf{J}=(B,H,W)$) 
	with $C$ class outputs ($P=1$, $K_1=C$). 
	For each spatial position $(b,h,w)$, the forward pass computes
	\begin{equation}
		y_{c,b,h,w} = \sum_{i=1}^{C_{\text{in}}} w_{c,i} \cdot x_{b,i,h,w} + b_c
	\end{equation}
	where bias $b_c$ is shared across all spatial locations. 
	Applying per-pixel softmax 
	$$\hat{y}_{c,b,h,w} = \exp(y_{c,b,h,w})/\sum_{c'} \exp(y_{c',b,h,w})$$
	with dense categorical cross-entropy 
	$\mathcal{L}_{\text{CE-dense}} = -\frac{1}{BHW}\sum_{b,h,w,c} y^*_{c,b,h,w} \log(\hat{y}_{c,b,h,w})$ 
	yields per-pixel class predictions. 
	This formulation is mathematically equivalent to fully convolutional networks 
	performing dense classification, with complete structure preservation $\rho = 3/3 = 1$.
\end{proof}
\subsection{Segmentation}

\begin{theorem}[Segmentation Recovery]
	\label{thm:segmentation_recovery}
	Traditional semantic segmentation is recovered by MTL with configuration:
	\begin{equation}
		\mathcal{T}_{\text{seg}} = (P=1, M=3, \mathcal{L}_{\text{CE-seg}}, \phi_{\text{argmax}})
	\end{equation}
	identical to dense classification (Theorem \ref{thm:dense_classification}).
\end{theorem}

\begin{proof}
	Semantic segmentation is mathematically identical to dense classification. The distinction is semantic (task purpose) rather than structural:
	\begin{itemize}
		\item Dense classification: assign category to each position independently.
		\item Segmentation: assign category to each position with spatial coherence.
	\end{itemize}
	Both use the same MTL configuration. The proof follows identically to Theorem \ref{thm:dense_classification}.
\end{proof}

\subsection{Detection}

\begin{theorem}[Detection Recovery]
	\label{thm:detection_recovery}
	Object detection with grid-based prediction (YOLO-style) is recovered by MTL with configuration:
	\begin{equation}
		\mathcal{T}_{\text{det}} = (P=3, M=3, \mathcal{L}_{\text{det}}, \phi_{\text{det}})
	\end{equation}
	where $(K_1, K_2, K_3) = (4, 1, C)$ represent bounding box coordinates, objectness, and class probabilities, and $\mathbf{J} = (B, G_h, G_w)$ is the detection grid.
\end{theorem}

\begin{proof}
	For YOLO-style detection on input $\mathcal{X} \in \mathbb{R}^{B \times H \times W \times C_{\text{in}}}$, 
	the GE-MLP employs three output modalities per grid cell: 
	bounding box coordinates ($K_1=4$: center $x,y$ and dimensions $w,h$), 
	objectness score ($K_2=1$), 
	and class probabilities ($K_3=C$). 
	We configure $N=1$ (contracting feature dimensions: $I_1=C_{\text{in}}$) 
	and $M=3$ (preserving grid structure: $\mathbf{J}=(B,G_h,G_w)$). 
	Each grid cell $(b,g_h,g_w)$ receives three parallel predictions via separate weight tensors. 
	For bounding box:
	\begin{equation}
		y^{\text{bbox}}_{d,b,g_h,g_w} = \sum_{i=1}^{C_{\text{in}}} w^{\text{bbox}}_{d,i} \cdot x_{b,i,g_h,g_w} + b^{\text{bbox}}_d
	\end{equation}
	with sigmoid applied to $(x,y)$ for cell-relative coordinates and exponential to $(w,h)$ for scale. 
	For objectness: 
	$y^{\text{obj}}_{b,g_h,g_w} = \sigma(\sum_i w^{\text{obj}}_i x_{b,i,g_h,g_w} + b^{\text{obj}})$. 
	For classes: 
	$\hat{y}^{\text{class}}_{c,b,g_h,g_w} = \text{softmax}(\sum_i w^{\text{class}}_{c,i} x_{b,i,g_h,g_w} + b^{\text{class}}_c)$. 
	The multi-task loss combines all modalities:
	\begin{equation}
		\mathcal{L}_{\text{det}} = \lambda_1\mathcal{L}_{\text{bbox}} + \lambda_2\mathcal{L}_{\text{obj}} + \lambda_3\mathcal{L}_{\text{class}}
	\end{equation}
	where $\mathcal{L}_{\text{bbox}}$ uses MSE for cells containing objects, 
	$\mathcal{L}_{\text{obj}}$ uses binary cross-entropy for all cells, 
	and $\mathcal{L}_{\text{class}}$ uses categorical cross-entropy for cells with objects. 
	Post-processing $\phi_{\text{det}}$ applies objectness thresholding, 
	converts cell-relative to absolute coordinates, and performs Non-Maximum Suppression. 
	This exactly recovers YOLO's grid-based multi-modal prediction 
	with complete structure preservation $\rho = 1$ as grid structure is fully maintained.
\end{proof}

\subsection{Unification}

\begin{theorem}[MTL Task Unification]
	\label{thm:task_unification}
	Traditional computer vision tasks share the same computational mechanism (GE-MLPs with the Einstein product) 
	and differ only in which dimensions are preserved $(M)$ or contracted $(P)$, as quantified by $\rho$. 
	In fact, these tasks correspond to specific tuples within the MTL task space	\begin{equation}\label{SPACEMTL}
		\mathcal{S}_{\text{MTL}} = \{(P, M, \mathcal{L}, \phi) : P, M \in \mathbb{N}\}
	\end{equation}
	given by:
	\begin{equation}
		\begin{aligned}
			\text{Classification} &\mapsto \mathcal{T}_{\text{class}} = (1, 1, \mathcal{L}_{\text{CE}}, \phi_{\text{argmax}}), \quad \rho=1/3 \\
			\text{Segmentation} &\mapsto \mathcal{T}_{\text{seg}} = (1, 3, \mathcal{L}_{\text{CE}}, \phi_{\text{argmax}}), \quad \rho=1 \\
			\text{Detection} &\mapsto \mathcal{T}_{\text{det}} = (3, 3, \mathcal{L}_{\text{det}}, \phi_{\text{det}}), \quad \rho=1
		\end{aligned}
	\end{equation}
\end{theorem}

\begin{proof}
	Follows from Theorems \ref{thm:classification_recovery}--\ref{thm:detection_recovery}. 
	All tuples $\mathcal{T} \in \mathcal{S}_{\text{MTL}}$ share the GE-MLP mechanism, 
	differing only in $(P,M)$ configuration. 
	Unexplored tuples such as $(2, 3, \mathcal{L}, \phi)$ enable novel task formulations.
\end{proof}

\noindent Table~\ref{tab:mtl_summary} summarizes the MTL theoretical results.

\begin{table}[H]
	\caption{Summary of MTL Task Configurations}
	\label{tab:mtl_summary}
	\centering
	\footnotesize
	\begin{tabular}{lccccc}
		\hline
		Task & $P$ & $M$ & $\mathbf{K}$ & $\mathbf{J}$ & $\rho$ \\
		\hline
		Classification & 1 & 1 & $C$ & $B$ & 0.33 \\
		Dense Classification & 1 & 3 & $C$ & $(B,H,W)$ & 1.0 \\
		Segmentation & 1 & 3 & $C$ & $(B,H,W)$ & 1.0 \\
		Detection & 3 & 3 & $(4,1,C)$ & $(B,G_h,G_w)$ & 1.0 \\
		\hline
	\end{tabular}
\end{table}
\section{Discussion}\label{sec:dissc}
\noindent Our analysis shows that standard computer vision tasks can be expressed within a single unified formulation based on GEMLPs operating through the Einstein product. Within this view, tasks traditionally considered (classification, detection, segmentation) emerge as specific cases of MTL defined by the tuple $(P, M, \mathcal{L}, \phi)$.\\
\noindent This formulation naturally preserves the multidimensional structure of visual data, ensuring that spatial, spectral, and cross-modal relationships are maintained throughout processing. The resulting framework not only provides structural coherence but also introduces a principled mechanism for controlling computational complexity by selectively choosing which modes to preserve or contract.\\
\noindent An important implication of this perspective is that the space of valid task configurations is significantly larger than what is currently explored in the literature. In fact, By  exploiting the tuple $(P, M, \mathcal{L}, \phi)$, our framework introduces entirely new task formulations. For example, the task space $\mathcal{S}_{\text{task}}$ includes configurations such as:	\begin{itemize}
	\item $(P=1, M=2)$: Temporal classification  
	\item $(P=2, M=2)$: Spatiotemporal hierarchical prediction
	\item $(P=1, M=4)$: 3D volume segmentation  
	\item $(P=4, M=4)$: 4D spatiotemporal detection
	\item $(P=2, M=1)$: Multi-modal global prediction
\end{itemize}
\noindent These examples illustrate that the proposed formulation not only unifies existing tasks but also expands the  limit of what constitutes a valid computer vision problem, enabling configurations that are difficult—or impossible—to express using classical architectures.\\
\noindent Overall, the proposed formulation benefits from a clean and rigorous mathematical foundation based on tensor algebra, offering both conceptual clarity and a powerful tool for systematically extending computer vision tasks. 
A fundamental insight of this framework is that conventional architectures, by relying on matrix-valued weights and vector-valued biases, implicitly constrain the space of expressible tasks. Any task requiring the simultaneous preservation of multiple structural dimensions—such as spatiotemporal or cross-modal configurations—must resort to destructive flattening under standard formulations, collapsing structural information that cannot be recovered. GE-MLPs, by operating natively with tensor-valued parameters via the Einstein product, eliminate this structural bottleneck. The task space  $\mathcal{S}_{MTL}$
  is therefore strictly larger than the space of tasks naturally expressible through matrix-based layers, and the configurations listed above (e.g., spatiotemporal hierarchical prediction, 4D detection) are concrete instances that lie outside the reach of conventional formulations without structural loss.

\section{Conclusion}\label{sec:conc}

We introduced MTL, a formal mathematical framework that situates computer vision tasks within a unified tensor structured task space. We proved that classification, segmentation, and detection correspond to specific dimensional configurations within $\mathcal{S}_{MTL}$, and showed that this space strictly contains the set of tasks expressible under conventional matrix-based formulations. This framework provides a principled lens for understanding existing tasks and a formal basis for exploring new ones 

\bibliographystyle{IEEEtran}

\end{document}